\documentclass{article}

\usepackage{amsmath, amsfonts, amssymb}

\usepackage{graphicx}        

\newcommand{\Dc}{{\mathcal D}}

\newcommand{\Nd}{\ensuremath{\mathbb{N}}}
\newcommand{\Pd}{\ensuremath{\mathbb{P}}}

\newcommand{\Rd}{\ensuremath{\mathbb{R}}}
\newcommand{\Sc}{\ensuremath{\mathcal{S}}}

\newenvironment{lemma}{Lemma}{}
\newenvironment{proof}{Proof}{}

\newcommand{\set}[1]{ \ensuremath{\left\{ #1 \right\}} }

\begin{document}

\title{Evolutionary Estimation of a Coupled Markov Chain Credit Risk Model\footnote{This research was partly supported by the Austrian National Bank Jubil{\"a}umsfond Project 12306.}}
\author{Ronald Hochreiter\footnote{Department of Statistics and Mathematics, WU Vienna University of Economics and Business, Augasse 2-6, A-1090 Vienna, Austria. {\tt ronald.hochreiter@wu.ac.at}} \and David Wozabal\footnote{Department of Business Administration, University of Vienna, Br{\"u}nner Stra\ss e 72, A-1210 Vienna, Austria. {\tt david.wozabal@univie.ac.at}}}

\maketitle

\abstract{There exists a range of different models for estimating and simulating credit risk transitions to optimally manage credit risk portfolios and products. In this chapter we present a Coupled Markov Chain approach to model rating transitions and thereby default probabilities of companies. As the likelihood of the model turns out to be a non-convex function of the parameters to be estimated, we apply heuristics to find the ML estimators. To this extent, we outline the model and its likelihood function, and present both a Particle Swarm Optimization algorithm, as well as an Evolutionary Optimization algorithm to maximize the likelihood function. Numerical results are shown which suggest a further application of evolutionary optimization techniques for credit risk management.}

\section{Introduction}

Credit risk is one of the most important risk categories managed by banks. Since the seminal work of \cite{Merton1974} a lot of research efforts have been put into the development of both sophisticated and applicable models. Furthermore, de facto standards like CreditMetrics and CreditRisk$^+$ exist. Numerous textbooks provide an overview of the set of available methods, see e.g. \cite{DuffieSingleton2003}, \cite{McNeilEtAl2005}, and \cite{Schoenbucher2003}. Evolutionary techniques have not yet been applied extensively in the area of credit risk management - see e.g. \cite{HagerS06} for credit portfolio dependence structure derivations or \cite{ZhangAS2010} for optimization of transition probability matrices. In this chapter, we apply the Coupled Markov Chain approach introduced by \cite{KaniovskiPflug2007} and provide extensions to the methods presented in \cite{HochreiterWozabal2009}. Section \ref{sec:cmc} briefly describes the Coupled Markov Chain model and its properties, and outlines the data we used for subsequent sampling. The likelihood function, which is to be maximized is discussed in Section \ref{sec:ml}. A non-trivial method to sample from the space of feasible points for the parameters is outlined in Section \ref{sec:sampling}. Two different evolutionary approaches to optimize the maximum likelihood function are presented: in Section \ref{sec:psa} a Particle Swarm Algorithm is shown, and Section \ref{sec:evolutionary} introduces an Evolutionary Optimization approach. Section \ref{sec:numres} provides numerical results for both algorithmic approaches, while Section \ref{sec:conclusion} concludes the chapter.

\section{Coupled Markov Chain Model}
\label{sec:cmc}

\subsection{Model Description}

In the Coupled Markov Chain model proposed in \cite{KaniovskiPflug2007} company defaults are modeled directly as Bernoulli events. This is in contrast to standard models used in the literature where indirect reasoning via asset prices is used to model default events of companies. The advantage of the proposed approach is that there are no heavy model assumptions necessary (normality of asset returns, no transaction costs, complete markets, continuous trading \ldots).

Portfolio effects in structured credit products are captured via correlations in default events. Companies are characterized by their current rating class and a second characteristic which can be freely chosen (industry sector, geographic area, \ldots). This classification scheme is subsequently used to model joint rating transitions of companies. We keep the basic idea of the standard Gaussian Copula model
$$ X = \rho \tau + (1-\rho) \phi, $$
where $\tau$ is the idiosyncratic part and $\phi$ is the systematic part determining the rating transition, while $0\leq \rho \leq 1$ is a relative weighting factor. More specifically the Coupled Markov Chain model can be described as follows: A company $n$ belongs to a sector $s(n)$ and is assigned to a rating class $X_n^t$ at time $t$ with $X_n^t \in \set{0, \ldots, M+1}$ and $t: \; 1\leq t \leq T$, with the credit quality decreasing with rating classes, i.e. $(M+1)$ being the default class, while 1 is the rating class corresponding to the best credit quality. The ratings of company $n$ are modeled as Markov Chains $X^t_n$. The rating process of company $n$ is determined by
\begin{itemize}
\item an idiosyncratic Markov Chain $\xi^t_n$.
\item a component $\eta^t_n$ which links $n$ to other companies of the same rating class.
\item Bernoulli switching variables $\delta^t_n$ which decide which of the two factors determines the rating, with $ \Pd( \delta^{t+1}_n = 1) = q_{s(n), X_n^t},$ i.e. the probability of success depends on sector and rating.
\end{itemize}
All the $\xi^t_n$ and $\delta^t_n$ are independent of everything else, while the $\eta_n^t$ have a non-trivial joint distribution modeled by common Bernoulli tendency variables $\chi_i, \; i: 1\leq i \leq M$, such that
$$ \Pd(\eta_n^t \leq X_n^t) = \Pd(\chi_{X_n^{t-1}} = 1) \mbox{ and }\Pd(\eta_n^t > X_n^t) = \Pd(\chi_{X_n^{t-1}} = 0),$$
i.e. the variables $\chi_i$ are indicators for a (common) non-deteriorating move of all the companies in rating class $i$. The rating changes of companies in different rating classes are made dependent by the non-trivial probability mass function $P_\chi: \set{0,1}^M \to \Rd$ of the vector $\chi = (\chi_1, \ldots, \chi_M)$.

The Coupled Markov Chain model is of the form:
$$X_n^t = \delta_n^t \xi_n^t + (1-\delta_n^t) \eta_n^t.$$
and exhibits properties, which are interesting for practical application. It takes a transition matrix $P=(p_{i,j})$ as input which governs the probability of transitions for $\xi^t_n$ and $\eta_i^t$, i.e.
$$ \mathbb{P}( \xi_n^t = j) = p_{m(n),j} \mbox{ and } \mathbb{P}( \eta_i^t = j) = p_{i, j}.$$
The model is capable of capturing different default correlations for different sectors and rating classes, and is able to give a more accurate picture of closeness to default than the standard model by including more than two states. The overall transition probabilities of $X_n$ again follow $P$, i.e. $$\mathbb{P}(X_n = j ) = p_{m(n), j}.$$

\subsection{Data}

Rating data from Standard \& Poors has been used, whereby $10166$ companies from all over the world have been considered. The data consists of yearly rating changes of these companies over a time horizon of $23$ years up to the end of $2007$. In total a number of $87.296$ data points was used. The second characteristic is the SIC industry classification code. Sectoral information has been condensed to six categories: Mining and Construction (1), Manufacturing (2), Transportation, Technology and Utility (3), Trade (4), Finance (5), Services (6). Likewise, rating classes are merged in the following way: AAA, AA $\rightarrow$ 1, A $\rightarrow$ 2, BBB $\rightarrow$ 3, BB, B $\rightarrow$ 4, CCC, CC, C $\rightarrow$ 5, D $\rightarrow$ 6. These clusters allow for a more tractable model by preserving a high degree of detail. The estimated rating transition probabilities from the data are shown in Tab. \ref{tab:transprob}.
\begin{table}
$$ P = \left( \begin{array}{cccccc}  0.9191& 0.0753& 0.0044& 0.0009& 0.0001& 0.0001 \\ 0.0335 & 0.8958 & 0.0657& 0.0036& 0.0006& 0.0009 \\ 0.0080  &  0.0674 &   0.8554  &  0.0665 & 0.0011  &  0.0016 \\ 0.0039  &  0.0092 &   0.0794  &  0.8678  &  0.0244  &  0.0153 \\ 0.0023    & 0.0034  &  0.0045  &  0.1759  &  0.6009  &  0.2131 \\ 0 & 0 & 0 & 0 & 0 & 1 \end{array} \right)$$
\caption{Estimated rating transition probabilities}
\label{tab:transprob}
\end{table}

\section{Maximum Likelihood Function}
\label{sec:ml}

The approach proposed by \cite{KaniovskiPflug2007} takes a Markov transition matrix $P = (p_{m_1, m_2})_{1\leq m_1, m_2 \leq (M+1)}$ as an input, i.e.
$$ \sum_{i=1}^{M+1} p_{i, m} = \sum_{i=1}^{M+1} p_{m,i} = 1, \quad \forall m: 1 \leq m \leq (M+1).$$

For $(M+1)$ rating classes, $N$ companies and $S$ industry sectors the parameters of the model are a matrix $Q=(q_{m,s})_{1 \leq s \leq S, \; 1\leq m \leq M}$ and a probability measure $P_\chi$ on $\set{0,1}^M$ satisfying some constraints dependent on $P$ (see problem \eqref{genProb}). Given rating transition data $X$ ranging over $T$ time periods we maximize the following monotone transformation of the likelihood function of the model
$$L(X; Q, P_{\chi}) = \sum_{t=2}^T log\left( \sum_{\bar{\chi} \in \{ 0, 1 \}^M} P_\chi( \chi^t = \bar{\chi} ) \prod_{s, m_1, m_2} f(x^{t-1}, s, m_1, m_2,; Q, P_\chi) \right)$$
with
$$ f(x^{t-1}, s, m_1, m_2,; Q, P_\chi) = \begin{cases} \left( \frac{q_{m_1,s}(p_{m_1}^+-1)  +1}{p_{m_1}^+} \right)^{I^t}, & m_1 \geq m_2, \; \bar{\chi}_{m_1} = 1\\
\left( \frac{q_{m_1,s}(p_{m_1}^- -1)  +1}{p_{m_1}^-} \right)^{I^t}, & m_1 < m_2, \; \bar{\chi}_{m_1} = 0 \\
q_{m_1, s}^{I^t}, & \mbox{otherwise}. \\
\end{cases}$$
where $I^t \equiv I^t(m_1, m_2, s)$ is a function dependent on the data $X$ which takes values in $\Nd$, $p_m^+ = \sum_{i=1}^m p_{m,i}$ and $p_m^- = 1- p_m^+$.

The above function is clearly non-convex and since it consists of a mix of sums and products this problem can also not be overcome by a logarithmic transform. Maximizing the above likelihood for given data $X$ in the parameters $P_\chi$ and $Q$ amounts to solving the following constrained optimization problem
\begin{eqnarray} \label{genProb}
\begin{array}{llll}
\max_{Q, P_\chi} & L(X; Q, P_{\chi}) & \\
s.t. 	& q_{m, s} & \in [0,1] \\
		& \sum_{\bar{\chi}: \bar{\chi}_i = 1} P_\chi(\bar{\chi}) &= p_{m_i}^+, & \forall i: 1 \leq i \leq M \\
		& \sum_{\bar{\chi}: \bar{\chi}_1 = 0} P_\chi(\bar{\chi}) &= 1 - p_{i}^+. \\
\end{array}
\end{eqnarray}

\section{Sampling Feasible Points}
\label{sec:sampling}

To sample from the space of feasible joint distributions for $\chi$ (i.e. the distributions whose marginals meet the requirements in \eqref{genProb}), first note that the distributions $P_\chi$ of the random variable $\chi$ are distributions on the space $\set{0,1}^M$ and therefore can be modeled as vectors in $\Rd^{2^M}$. To obtain samples we proceed as follows.

\begin{enumerate}

	\item To get a central point in the feasible region, we solve the following problem in dependence of a linear functional $\Psi: \Rd^{2^M} \to \Rd$.
		\begin{eqnarray} \label{startProb}
		\begin{array}{llll}
		\max_{P_\chi} & \Psi(P_\chi) &  \\
		s.t. 	& q_{m, s} & \in [0,1] \\
				& \sum_{\bar{\chi}: \bar{\chi}_i = 1} P_\chi(\bar{\chi}) &= p_{m_i}^+, & \forall i: 1 \leq i \leq M \\
				& \sum_{\bar{\chi}: \bar{\chi}_1 = 0} P_\chi(\bar{\chi}) &= 1 - p_{i}^+. \\		\end{array}
		\end{eqnarray}
	and call the solution set $\Sc(\Psi)$. By generating linear $\Psi$ functionals with random coefficients and picking $x^+ \in \Sc(\Psi)$ and $x^- \in \Sc(-\Psi)$ we get vertices of the feasible set of distributions for $\chi$ modeled as a polyhedron in $\Rd^{2^M}$. In this way we generate a set of vertices $V$ for the feasible region $\Omega$ of the above problem. Note that to enforce all the constraints in \eqref{startProb} we need $M+1$ linear equality constraints which describe a $2^M - (M+1)$ dimensional affine subspace in $\Rd^{2^M}$.
	
	\item Get a central point in $c\in \Omega$ by defining $$ c = \frac{1}{|V|} \sum_{v \in V} v.$$
	
	\item Sample $K \in \Nd$ directions of unit length from a spherical distribution (like the multivariate standard normal with independent components) in $\Rd^{2^M - M -1}$ to get uniformly distributed directions. Map these directions to $\Rd^{2^M}$ using an orthogonal basis of the affine subspace $A$ of $\Rd^{2^M}$ described by the last set of constraints in \eqref{startProb} to obtain a set of directions $\Dc$ in $A$.
	
	\item For every $d \in \Dc$ determine where the line $c + \lambda d$ meets the boundary of the feasible set of \eqref{startProb}. Call the length of the found line segment $l_d$.
	
	\item Fix a number $L \in \Nd$ and sample
	$$\left\lceil \frac{l_d}{\sum_{d\in \Dc} \bar{l}_d} L \right\rceil$$
	points on the line $l_d$. In this way we get approximately $KL$ samples for $P_\chi$.
	
\end{enumerate}

Contrary to obtaining samples from $P_\chi$, getting suitable samples for $Q$ is fairly easy, since all the components are in $[0,1]$ and independent of each other. Note that both the set of feasible matrices $Q$ as well as the set of feasible measures $P_\chi$ are convex sets in the respective spaces.

\section{Particle Swarm Algorithm}
\label{sec:psa}

In the following we give a brief description of the Particle Swarm Algorithm (PSA), which follows the ideas in \cite{KennedyEberhart1995}.

\begin{enumerate}
	\item Choose $\delta > 0$ and $S \in  \Nd$.
	
	\item Generate $S$ permissible random samples $x_k = (Q^k, P_\chi^k)$ for $k=1, \ldots, S$ as described above, i.e. $q_{s,m}^k \in [0,1]$ and $P_\chi$ is consistent with the constraints in \eqref{startProb}. Each sample is a particle in the algorithm. Set $\hat{x}_k = x_k$ and $v_k = 0$ for all $k=1, \ldots, S$.

	\item Set $\hat{g} \leftarrow \arg \min_k L(x_k)$.

	\item For all particles $x_k$
		\begin{enumerate}
			\item Let the particles fly by first computing a velocity for the $k$-th particle
			\begin{equation} \label{psaFly}
			v_k \leftarrow c_0 v_k + c_1 r_1 \circ (\hat{x}_k - x_k) + c_2 r_2 \circ (\hat{g} - x_k)
			\end{equation}
			where $c_0$, $c_1$, $c_2$ are fixed constants, $r_1$ and $r_2$ are random matrices (component-wise uniform) of the appropriate dimension and $\circ$ is the Hadamard matrix multiplication.  Then a new position for the particle is found by the following assignment
			$$ x_k \leftarrow x_k + v_k.$$

			\item If $L(x_k) > L(\hat{x}_k)$ then $\hat{x}_k \leftarrow x_k$.
		\end{enumerate}
	\item $L(x_k) > L(\hat{g})$ for some $x_k$, then $\hat{g} \leftarrow x_k$.

	\item If $var(L(x_k))< \delta$ terminate the algorithm, otherwise go to step 3.
\end{enumerate}
The main idea is that each particle $k$ knows its best position $\hat{x}_k$ as of yet (in terms of the likelihood function) and every particle knows the best position ever seen by any particle $\hat{g}$. The velocity of the particle changes in such a way that it is drawn to these positions (to a random degree). Eventually all the particles will end up close to one another and near to a (local) optimum.

Note that in step 4(a) of the above algorithm, a particle may leave the feasible region either by violating the constraints on $P_\chi$ or $Q$. In this case the particle \emph{bounces of the border} and completes it's move in the modified direction. To be more precise: the particles can only leave the feasible region, by violating the constraints that either elements of $Q$ or probabilities assigned by $P_\chi$ are no longer in $[0,1]$ (the last two constraints in \eqref{startProb} can not be violated since the particles only move in the affine subspace of $\Rd^{2^M}$ where these constraints are fulfilled).

If $x_k^i + v_k^i > 1$ for some $1 \leq i \leq 2^{2^M} + MS$ (the case where $x_k^i + v_k^i <0$ works analogously), determine the maximum distance $\lambda$ that the particle can fly without constraint violation, i.e. set $$\lambda = \frac{(1-x_k^i)}{v_k^i}$$
and set $x^k \leftarrow x^k + \lambda v^k$. Now set $\bar{v}_k$ such that the new velocity makes the particle bounce off the constraint \emph{as would be expected} and make the rest of the move, i.e. set
$$ x^k \leftarrow x^k + (1-\lambda) \bar{v}^k.$$
In the case that the violation concerns an element of $Q$ the modification only concerns a change of sign, i.e.  $\bar{v}_k^i \leftarrow -v_k^i$ and $\bar{v}_k^j \leftarrow v_k^j$ for all $j \neq i$. 

In the following we describe how to find a \emph{bounce off} direction, if a constraint on an element of $P_\chi$ is violated: first determine the hyperplane $H$ in $\Rd^{2^M}$ that represents the constraint. Notice that $H = \set{ x \in \Rd^{2^M}: x_i = 1}$ for some $i$. We use the following Lemma to get a representation of the hyperplane $H$ in the affine subspace $A$.

\begin{lemma}
Let $A$ be an affine subspace of $\Rd^D$ with orthonormal basis $e_1, \ldots, e_d$ with $D<d$ and $H$ a hyperplane with normal vector $n$. The normal vector of the hyperplane $A\cap H$ in $A$ is $$\bar{n} = \sum_{i=1}^d \langle e_i, n \rangle e_i.$$
\end{lemma}
\begin{proof}
Let $H = \set{ x \in \Rd^D: \langle x, n \rangle = c}$ for some $c \in \Rd$, then
\begin{eqnarray*}
c &=& \langle y, n \rangle = \left\langle \sum_{i=1}^d \langle y, e_i \rangle e_i, n \right\rangle = \sum_{i=1}^d \langle y, e_i \rangle \langle e_i, n \rangle \\
&=& \left\langle y, \sum_{i=1}^d \langle e_i, n \rangle e_i \right\rangle = \langle y, \bar{n} \rangle.
\end{eqnarray*}
This implies that the point $y \in A \cap H$, iff $\langle y, \bar{n} \rangle = c$.
\end{proof}

Using the above Lemma we identify the normal vector $\bar{n} \in A$ of the hyperplane $\bar{H} = H \cap A$ in $A$. Without loss of generality we assume that $|| \bar{n}|| = 1$. Now use Gram-Schmidt to identify a orthonormal system $\bar{n}, y_2, \ldots, y_{2^{2^M} - (M+1)}$ in $A$ and represent $v_k$ as
$$ v_k = \langle \bar{n}, v_k \rangle \bar{n} + \sum_{i \geq 2} \langle y_i, v_k \rangle y_i.$$
The transformed velocity can now be found as
$$ \bar{v}_k = -\langle \bar{n}, v_k \rangle \bar{n} + \sum_{i \geq 2} \langle y_i, v_k \rangle y_i.$$

Obviously an implementation of the algorithm has to be able to handle multiple such bounces in one move (i.e. situation where the new direction $\bar{v}_k$ again leads to a constraint violation). Since the details are straightforward and to avoid too complicated notation, we omit them here for the sake of brevity.

\section{Evolutionary Algorithm}
\label{sec:evolutionary}

Evolutionary algorithms are well suited to handle many financial and econometric applications, see especially \cite{BrabazonONeill2008}, \cite{BrabazonONeill2009}, and \cite{BrabazonONeill2006} for a plethora of examples.

Each chromosome consists of a matrix $Q$ and a vector $P_\chi$. While the parameters in $Q$ can be varied freely between $0$ and $1$, and the parameters $P_\chi$ do need to fulfill constraints (see above), the genetic operators involving randomness are mainly focused around the matrix $Q$. Therefore, four different genetic operators are used:
\begin{itemize}
\item {\bf Elitist selection.} A number $e$ of the best chromosomes are added to the new population.
\item {\bf Intermediate crossover.} $c$ intermediate crossovers (linear interpolation) between the matrix $Q_1$ and $Q_2$ of two randomly selected parents are created using a random parameter $\lambda$ between $0$ and $1$, i.e. two children $Q_3, P_{\chi,3}$ and $Q_4, P_{\chi,4}$ are calculated as follows:
$$Q_3 = \lambda Q_1 + (1-\lambda) Q_2, P_{\chi,3} = P_{\chi, 1},$$
$$Q_4 = (1-\lambda) Q_1 + \lambda Q_2, P_{\chi,4} = P_{\chi, 2}.$$
\item {\bf Mutation.} $m$ new chromosomes are added by mutating a randomly selected chromosome from the parent population, and adding a factor $\phi$ in the range $[-0.5,0.5]$ to the matrix $Q$. The values are truncated to values between 0 and 1 after the mutation.
\item {\bf Random additions.} $r$ random chromosomes are added with a random matrix $Q$ and a randomly selected vector $P_\chi$ from the parent population.
\end{itemize}

\section{Numerical Results}
\label{sec:numres}
Both algorithms were developed in MatLab R2007a, while the linear problems \eqref{startProb} were solved using MOSEK 5. A stability test has been conducted to validate the results of both optimization algorithms: the maximum (pointwise) differences of parameter estimates $P_\chi$ and $Q$ between the different optimization runs is used to verify that these important parameters, which are e.g. used for a credit portfolio optimization procedure, do not differ significantly.

\subsection{Particle Swarm Algorithm}
The parameters in \eqref{psaFly} were set to $c_0 = 0.5$, $c_1 = 1.5$ and $c_2 = 1.5$. The algorithm was made to complete $150$ iterations with around $200$ initial samples (where the $\chi$ are simulated on $40$ lines with approximately $5$ samples each). To test the stability of the algorithm $50$ runs of the algorithm were performed. As can be clearly observed in Fig. \ref{fig:psaresults1} the likelihood of the best particle as well as the mean likelihood of the swarm converges nicely and stabilizes around iteration $25$.

\begin{figure}
\begin{center}
\scalebox{0.2}{
\includegraphics{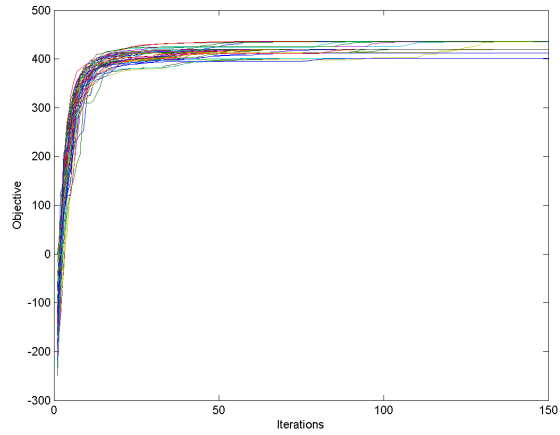}
}
\caption{Objective function of the PSA algorithm: maximum per iteration.}
\label{fig:psaresults1}
\end{center}
\end{figure}

\begin{figure}
\begin{center}
\scalebox{0.2}{
\includegraphics{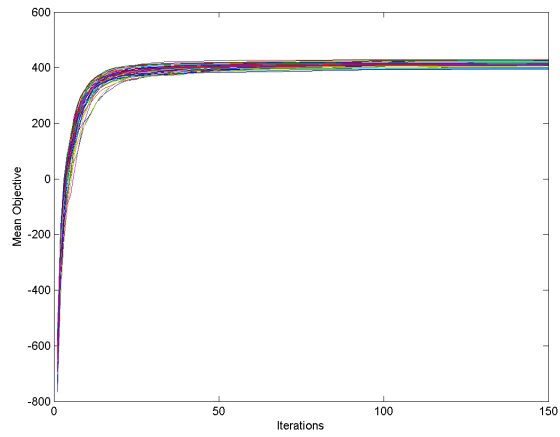}
}
\caption{Objective function of the PSA algorithm: population mean.}
\label{fig:psaresults1b}
\end{center}
\end{figure}

Each run took around $1$ hour to complete $150$ iterations. Stability results are shown in Fig. \ref{fig:psaresults2}.

\begin{figure}
\begin{center}
\scalebox{0.2}{
\includegraphics{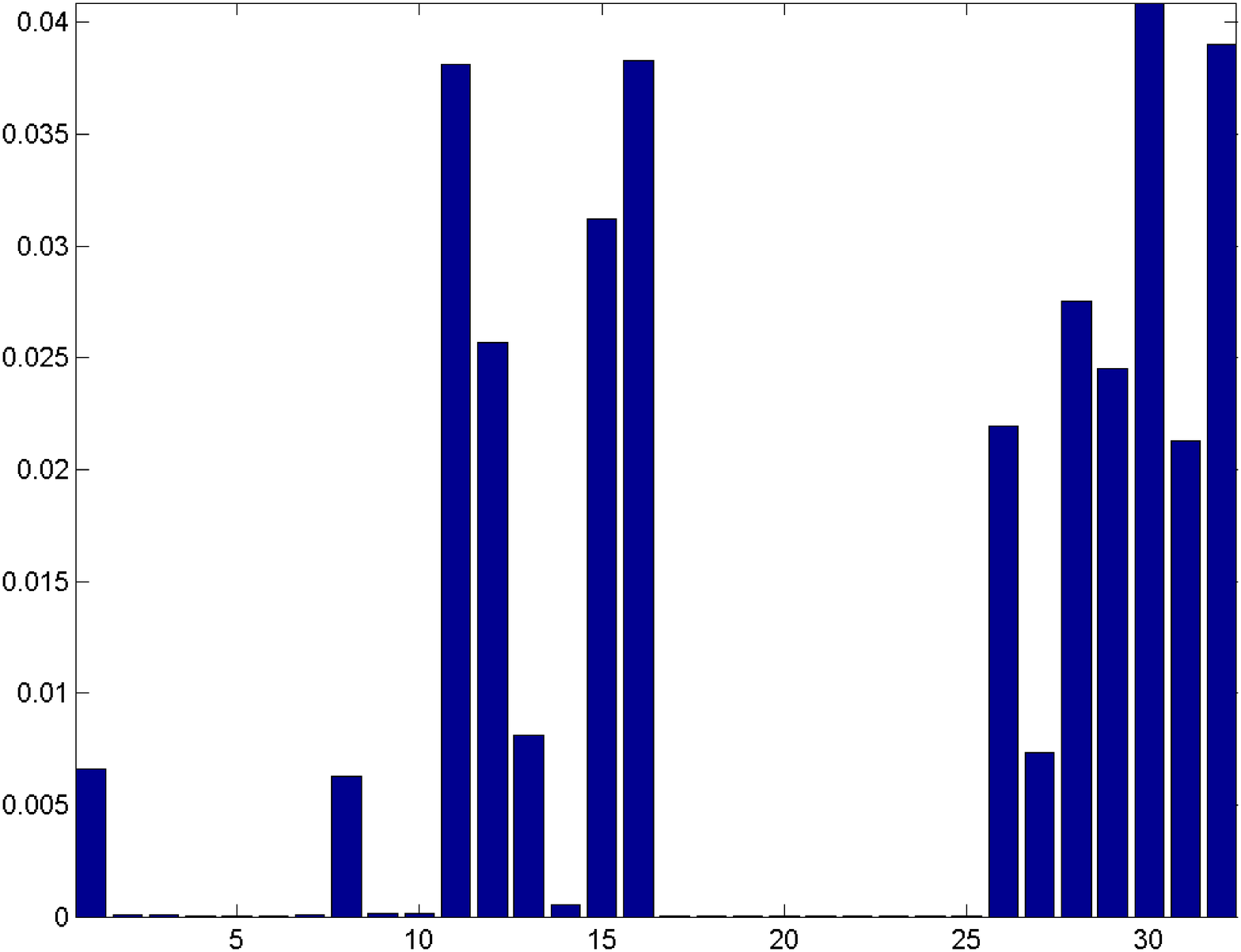}
}
\caption{Maximum (pointwise) differences of parameter estimates $P_\chi$ for different runs for the PSA.}
\label{fig:psaresults2}
\end{center}
\end{figure}

\begin{figure}
\begin{center}
\scalebox{0.2}{
\includegraphics{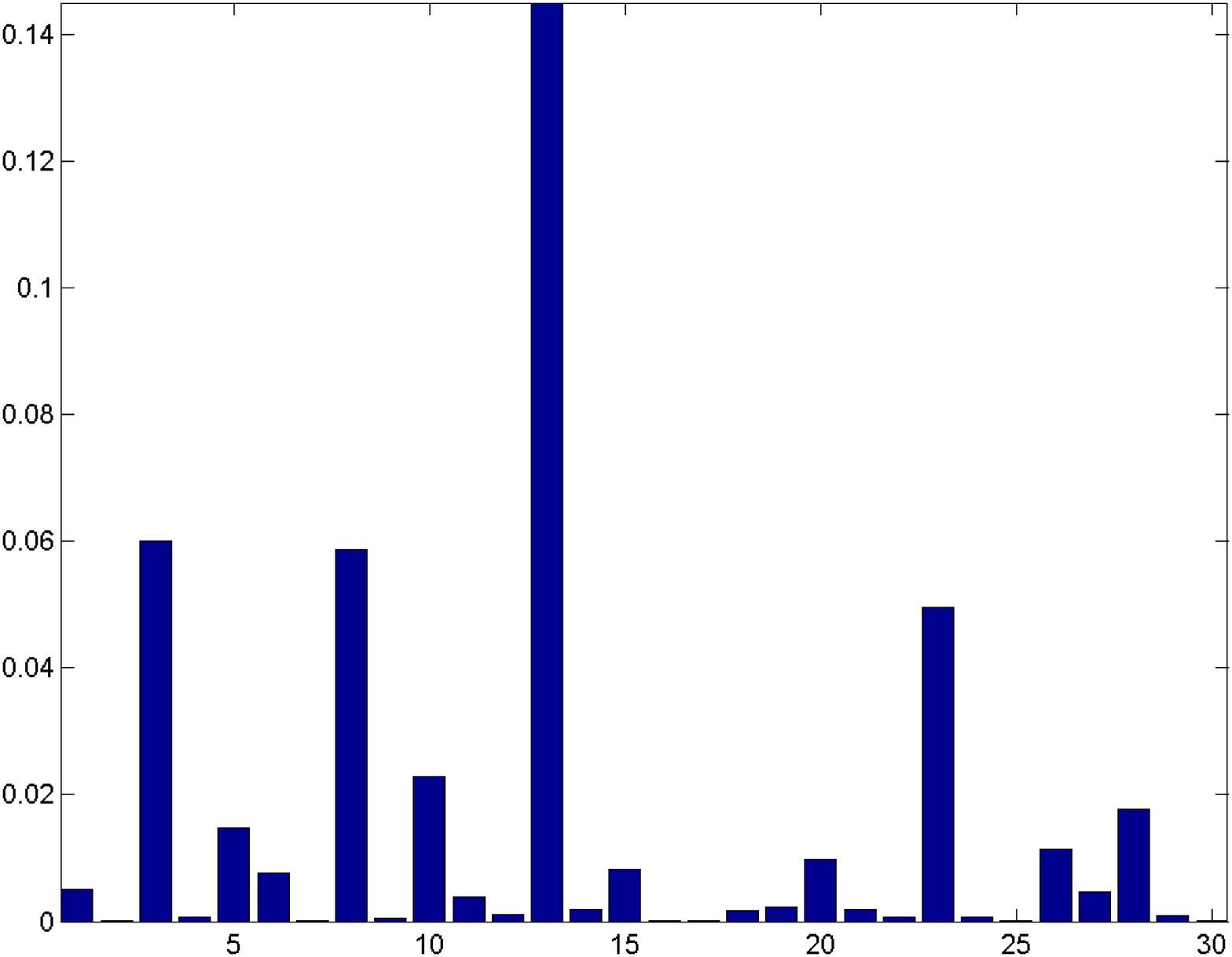}
}
\caption{Maximum (pointwise) differences of parameter estimates $Q$ for different runs for the PSA.}
\label{fig:psaresults2b}
\end{center}
\end{figure}

The variances of the populations in every iterations are plotted in Figure \ref{fig:psaVar}. Since the variances sharply increase from very high initial values and in most cases drop to rather low values quickly the plot depicts the variance after applying a logarithmic transformation (base 10) as well as the mean variance over all the runs. While in most runs the variances decreases from values of the magnitude $10^5$ to the range of $10^3$, some of the runs end up with significantly lower and higher variances. The latter being a sign that the PSA sometimes fails to converge, i.e. the particles do not concentrate at one point after the performed number of iterations.
However, this is not problematic since we are only interested in the likelihood of the best particle which seems to be pretty stable at around 400. The results depicted in Figure \ref{fig:psaresults2} confirm, that despite the high variance in some of the runs the likelihood of the best particle remains stable for all the runs.

\begin{figure}
\begin{center}
\scalebox{0.2}{\includegraphics{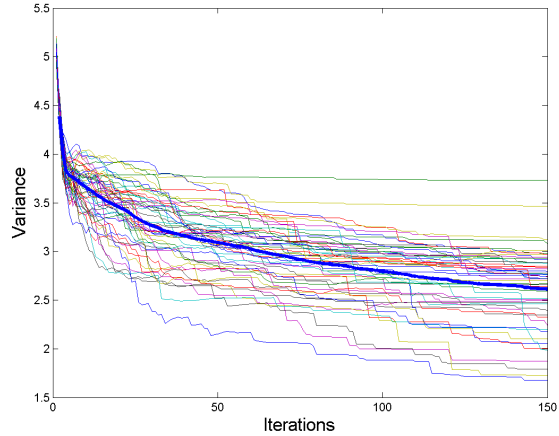}}
\caption{Variances of the objective values of the swarm (variances are transformed with $x \mapsto log_{10}(x)$ for better interpretability of the results. The mean (logarithmic) variance is depicted by the bold line.}
\label{fig:psaVar}
\end{center}
\end{figure}

\subsection{Evolutionary Algorithm}

The following parameters have been used to calculate results: The number of maximum iterations has been set to $150$. Each new population consists of $e = 30$ elitist chromosomes, $c=50$ intermediate crossovers, $m=100$ mutations, and $r=50$ random additions. $50$ runs have been calculated, and $10$ tries out of these are shown in Fig. \ref{fig:evoresults1} - both the maximum objective value per iteration (left) as well as the population mean (right). Due to the high number of random additions, mutations and crossovers, the mean is relatively low and does not significantly change over the iterations, which does not influence the results. The initial population size were $750$ randomly sampled chromosomes, independently sampled for each try. It can be clearly seen, that due to the different algorithmic approach, the convergence is different from the PSA.

\begin{figure}
\begin{center}
\scalebox{0.2}{
\includegraphics{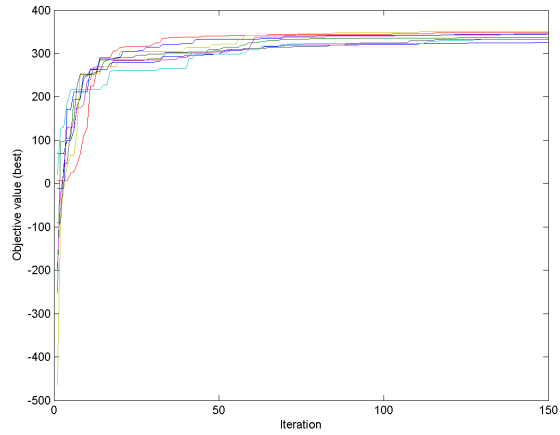}
}
\caption{Objective function of the EA: maximum per iteration.}
\label{fig:evoresults1}
\end{center}
\end{figure}

\begin{figure}
\begin{center}
\scalebox{0.2}{
\includegraphics{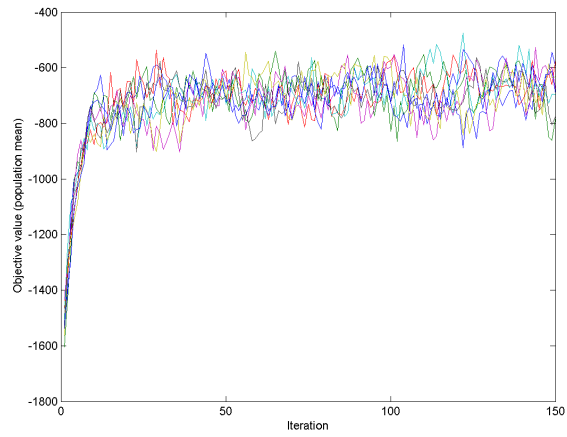}
}
\caption{Objective function of the EA: population mean.}
\label{fig:evoresults1b}
\end{center}
\end{figure}

Each run took approximately $70$ minutes to complete the $150$ iterations. Stability results are shown in Fig. \ref{fig:evoresults2}. The population variance is shown in Fig. \ref{fig:evoresults3}, and clearly exhibits a different behavior than the PSA algorithm as expected.

\begin{figure}
\begin{center}
\scalebox{0.2}{
\includegraphics{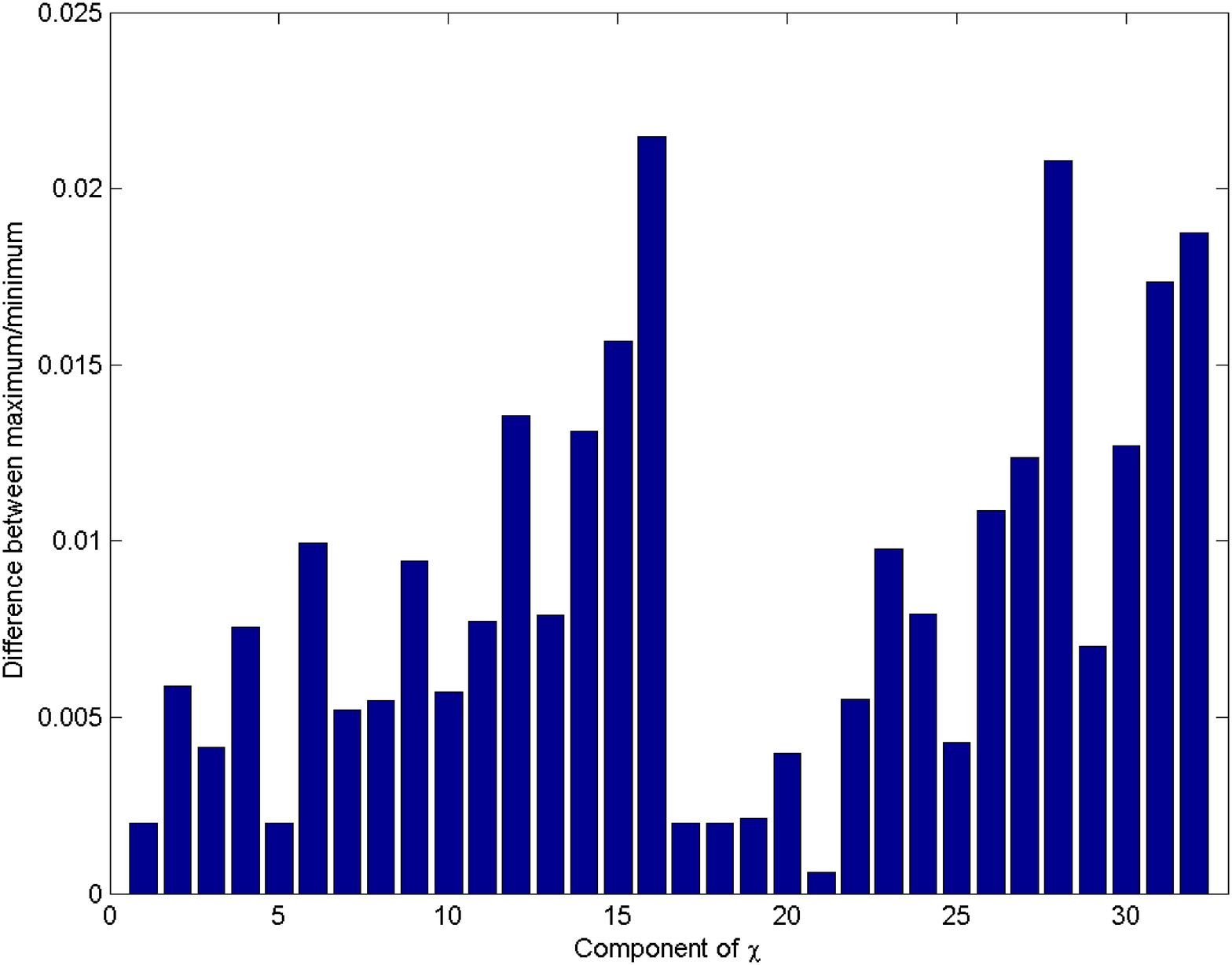}
}
\caption{Maximum (pointwise) differences of parameter estimates $P_\chi$ for different runs for the EA.}
\end{center}
\label{fig:evoresults2}
\end{figure}

\begin{figure}
\begin{center}
\scalebox{0.2}{
\includegraphics{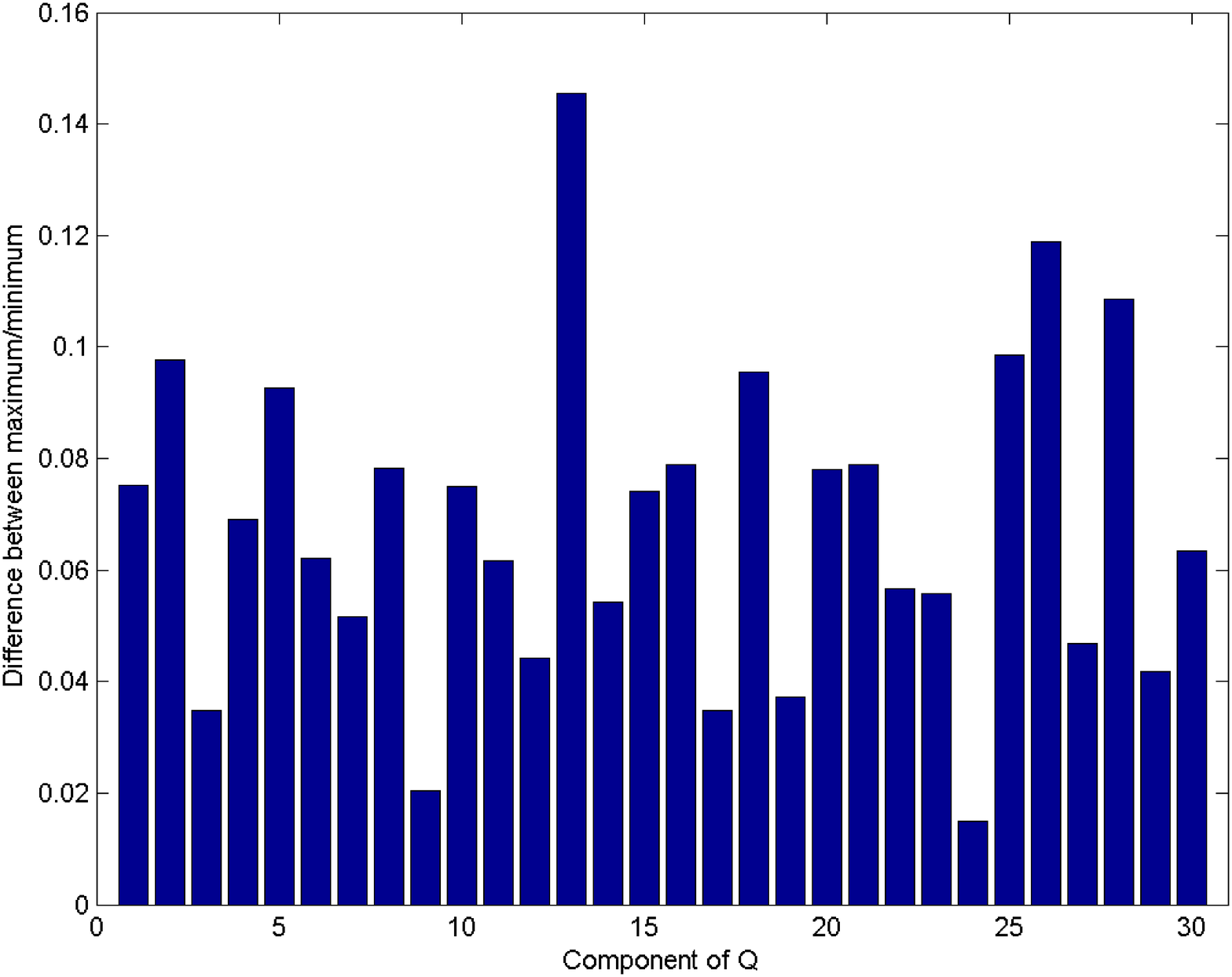}
}
\caption{Maximum (pointwise) differences of parameter estimates $Q$ for different runs for the EA.}
\label{fig:evoresults2b}
\end{center}
\end{figure}

\begin{figure}
\begin{center}
\scalebox{0.2}{
\includegraphics{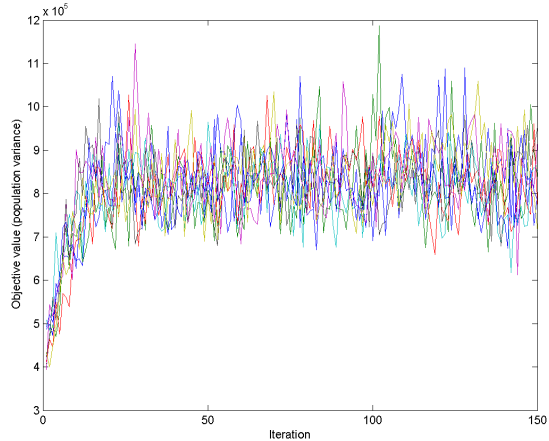}
}
\caption{Population variance of the EA.}
\label{fig:evoresults3}
\end{center}
\end{figure}

\subsection{Comparison of Methods}
\label{sec:comparison}

Comparing the results of the current implementations of the two optimization heuristics the results found by the PSA consistently yield a higher objective value than the solutions obtained with the EA (for the best particle/chromosome as well as for the mean). The computing time for the two methods is similar and is mainly used for the expensive objective function evaluations. Furthermore the presented computational evidence shows the typical behavior of the variance given the two heuristic optimization techniques. While the PSA generally performs slightly better than the EA, it might well be that it gets stuck in a local optimum, which might be avoided using the EA. One can see from the figures that the maximum difference between the estimated parameters for different runs are smaller on average for the PSA. However, the analysis of the distribution of these differences reveals the interesting fact, that while for the EA the differences are more uniform in magnitude and the highest as well as the lowest deviations can be observed for the PSA. With the realistically sized data set both methodologies are well suited and the final choice is up to the bank or company which implements and extends the presented method, i.e. has to be based on the expertise available.

\subsection{Application of the Model}
\label{sec:application}

Once the Coupled Markov Chain model has been estimated using evolutionary techniques shown above, it can be used to simulate rating transition scenarios for different sets of companies, which allows for pricing and optimization of various structured credit contracts like specific CDX tranches, e.g. a Mean-Risk optimization approach in the sense of \cite{Markowitz1987} can be conducted for which evolutionary techniques can be used again as shown by e.g. \cite{Hochreiter07} and \cite{Hochreiter08}, such that a whole credit risk management framework based on evolutionary techniques can be successfully implemented.

\section{Conclusion}
\label{sec:conclusion}

In this Chapter, we presented the likelihood function for a Coupled Markov Chain model for contemporary credit portfolio risk management. We presented two different heuristic approaches for estimating the parameter of the likelihood function. Both are structurally different, i.e. the population mean of each method differs significantly. However, both are valid approaches to estimate parameters. Once the parameters are estimated, many applications are possible. One prominent example is to generate scenarios for future payment streams implied by an existing portfolio of Credit Default Swap Indices (CDX) by Monte Carlo simulation. This allows for assessing the risk of the current position and price products which might be added to the portfolio in the future and thereby determine their impact on the overall exposure.

\bibliographystyle{plain}
\bibliography{evocmcbook}

\end{document}